\RequirePackage{amsmath}
\documentclass[runningheads]{llncs}
\usepackage{lmodern}
\usepackage[utf8]{inputenc}
\usepackage{latexsym}
\usepackage[T1]{fontenc}
\usepackage{amssymb}
\usepackage{stmaryrd}
\usepackage{eulervm}
\usepackage{palatino}
\usepackage{mdwmath}
\usepackage{mdwtab}

\usepackage{graphicx}
\usepackage[cmyk]{xcolor}
\usepackage{tikz}
\usepackage{tabularx}
\usepackage{array}
\usepackage{eqparbox}
\usepackage{booktabs}

\usepackage{amssymb}
\usepackage{makecell}
\usepackage{diagbox}

\usepackage{array}
\usepackage{lineno}

\usepackage{url}
\usepackage{hyperref}
\usepackage{enumitem}

\newcommand{\bsi}{\boldsymbol{\beth}}

\usepackage{enumitem}

\title{Algebraic Models for Qualified Aggregation in General Rough Sets, and Reasoning Bias Discovery}
\titlerunning{Algebraic Models for Aggregation}
\author{\textsf{A Mani}\thanks{This research is supported by Woman Scientist Grant No. WOS-A/PM-22/2019 of the Department of Science and Technology.}}
\authorrunning{A Mani}
\institute{Machine Intelligence Unit, Indian Statistical Institute, Kolkata\\
203, B. T. Road, Kolkata-700108, India\\
Email: \texttt{$a.mani.cms@gmail.com$} \texttt{$amani.rough@isical.ac.in$}\\
Homepage: \url{https://www.logicamani.in}\\
Orcid: \url{https://orcid.org/0000-0002-0880-1035} }

\begin{document}

\maketitle

\begin{abstract}
In the context of general rough sets, the act of combining two things to form another is not straightforward. The situation is similar for other theories that concern uncertainty and vagueness. Such acts can be endowed with additional meaning that go beyond structural conjunction and disjunction as in the  theory of $*$-norms and associated implications over $L$-fuzzy sets. In the present research, algebraic models of acts of combining things in generalized  rough sets over lattices with approximation operators (called rough convenience lattices) is invented. The investigation is strongly motivated by the desire to model skeptical or pessimistic, and optimistic or possibilistic aggregation in human reasoning, and the choice of operations is constrained by the perspective. Fundamental results on the weak negations and implications afforded by the minimal models are proved. In addition, the model is suitable for the study of discriminatory/toxic behavior in human reasoning, and of ML algorithms learning such behavior.
\end{abstract}


\section{Introduction}

In any context, the act of combining two things involves meta-level or semantic assumptions. These are more involved in the context of general rough sets because of the increased complexity of associated domains of discourse. Any generalized conjunction-like operation is referred to as an \emph{aggregation}, while a generalized disjunction-like operation as a \emph{co-aggregation} (valuations are not assumed). The situation is similar for other mature theories that concern uncertainty, vagueness or imprecision. For example, the theory of $*$-norms and associated implications are extensively investigated over $L$-fuzzy sets \cite{bfpr2013,bmjb2008}. The purpose of the present research is to invent (or construct), and investigate models of somewhat related acts of combining things over lattices with approximation operators, and without explicit negations. At the application level, this research is strongly motivated by the desire to model skeptical or pessimistic or cautious, and overly optimistic reasoning over concepts (in human reasoning), and the choice of operations is constrained by the perspective. 

Fundamental semantic results that formally address the following constraint on domains of discourse: \textsf{When things are implied and negated in some perspectives then they are being approximated and vice versa} are proposed. More specifically, concrete algebraic models (that involve a surprisingly weak set of axioms) in which the principle is valid are shown to exist in this research. 

Several algebraic models of the operations of combining objects or rough objects (in several senses) in the context of classical, general and granular rough sets are known \cite{am501,amedit,ppm2,jj}. However, not many impose a meaning constraint that amount to combining in skeptical or biased or bigoted or overly optimistic ways. These concepts can possibly be attained relatively through partial orders on approximations. For example, if the lower approximation $l_1$ approximates better than another lower approximation $l_2$, then the latter is relatively more skeptical than the former. Consequently, aggregations of the $l_2$-approximations of objects must be more skeptical than that of aggregations of $l_1$-approximations. In the present research, the relative aspect is hidden because in most cases, the collections of rough objects (in various senses) form a lattice.

In the context of Pawlak/classical rough sets, it is proved by the present author \cite{am9411} that aggregations $f$ (interpreted as \emph{rough dependence}), defined by the equation  $f(a, b) = a^l \cap b^l$ can be used to define algebraic models that make no reference to approximations. The intent in the paper was to establish the differences between rough sets and probability theory from a \emph{dependence} perspective. For a fuller discussion, the reader is referred to Section \ref{skpd}. However, the status of this operation and related ones in the abstract rough set literature is not known. This fundamental problem is solved in this research in suitably minimalist frameworks without negation operations. The framework is far weaker than the general approximation algebras considered in the paper \cite{cd3}, and is a specific version of a high general granular operator space \cite{am501,am5586} without the granularity requirement. Specific set-theoretic subclasses of high general granular operator spaces are also covered. 

The organization of this paper is as follows. Necessary background is outlined in the next section. The model(s) are invented in the third section. Illustrative concrete and abstract instances of the models are explored in the section on examples. Connections between qualified aggregation and rough dependence are clarified in the fifth section.

\section{Background}

\begin{definition}
A \emph{L-Fuzzy set} \cite{gj1967} is a map $\varphi : X \longmapsto L$, where $X$ is a set and $L = \left\langle \underline{L}, \leq \right\rangle$ is a quasi-ordered set. The set of all L-fuzzy sets will be denoted by $\mathbb{F}(X, L)$.
\end{definition}

\subsection{T-Norms, S-Norms, Uninorms, and Implications}

While T-norms, S-norms and uninorms are primarily viewed as operations on lattices or partially ordered sets for the algebraic and logical models of $L$-fuzzy sets, they can be utilized in the algebraic models of entirely different phenomena. The mentioned T/S/uni-norms are well-known algebraic operations in the algebra literature because the topological constraints of the unit interval context are not imposed. The conditions that make related considerations stand out from those on corresponding order-compatible algebras are boundedness and the role of additional operations such as those of generalized implications and negations. Some essential concepts (for more details, see for example \cite{bsmm2022,bfpr2013,yrr2004,ppm2,mkcs}) are mentioned here. 

\begin{definition}
Let $P= \left\langle \underline{P},\leq, e  \right\rangle$ be a partially ordered set (poset) on the set $\underline{P}$ with distinguished element (or $0$-ary operation) $e$, then any order-compatible binary associative operation $\cdot$ on it with identity element $\cdot$ is referred to as a \emph{pseudo uni-norm}. A commutative pseudo uni-norm is a \emph{uni-norm}. If $e$ is the greatest (respectively least) element of $P$ then $\cdot$ is a \emph{pseudo t-norm} (respectively \emph{pseudo s-norm}).
\end{definition}

The set of all pseudo uninorms on the poset $P$ is denoted by $\mathcal{U}(P, e)$. It forms a poset in the induced point-wise order. Both t-norms and t-conorms are uninorms.

\begin{definition}
If $L$ is a bounded lattice with bottom $\bot$ and top $\top$, then a \emph{t-norm} $\odot$ is a commutative, associative order compatible  monoidal operation with $\top$ being the identity. A \emph{s-norm} (or t-conorm) is a commutative, associative order compatible  monoidal operation with $\bot$ being the identity.
\end{definition}

Consider the conditions possibly satisfied by a map $n: L\longmapsto L$:
\begin{align}
n(\bot) = \top \, \&\, n(\top) = \bot \tag{N1}\\
(\forall a, b) (a\leq b \longrightarrow n(b)\leq n(a)) \tag{N2}\\
(\forall a) n(n(a)) = a \tag{N3}\\
n(a) \in \{\bot , \top \} \text{ if and only if } a= \bot \text{ or } a = \top \tag{N4}
\end{align}

$n$ is a \emph{negation} if and only if it satisfies \textsf{N1} and \textsf{N2}, while $n$ is a \emph{strong negation} if and only if it satisfies all the four conditions.

\subsection{Implication operations} 

Implications satisfy a wide array of properties as they depend on the other permitted operations. Here some relevant ones are mentioned. 

A function $\bsi : L^2 \longmapsto L $ is an \emph{implication} \cite{bfpr2013} if it satisfies (for any $a, b, c \in L$) the following:
\begin{align}
\text{If } a \leq b \text{ then } \bsi bc \leq \bsi ac \tag{First Place Antitonicity FPA}\\
\text{If } b \leq c \text{ then } \bsi ab \leq \bsi ac \tag{Second Place Monotonicity SPM}\\
\bsi \bot\bot =\top \tag{Boundary Condition 1: BC1}\\
\bsi \top \top = \top \tag{Boundary Condition 2: BC2}\\
\bsi \top \bot = \bot  \tag{Boundary Condition 3: BC3}
\end{align}

Infix notation is preferred for algebraic reasons. The set of all implications on the lattice $L$ will be denoted by $\mathcal{I}(L)$. It can be endowed with a bounded lattice structure under the induced order
\[\bsi_1 \preceq \bsi_2 \text{ if and only if } (\forall a, b\in L ) \bsi_1 ab \leq \bsi_2 ab .\]

The top $\bsi_\top$ and bottom $\bsi_\bot$ implications are defined as follows:
\begin{itemize}
\item {If $a = \top \,\& \,b= \bot$ then $\bsi_\top ab = \bot$,  otherwise $\bsi_\top ab = \top$.}
\item {If $a = \bot \,\& \,b= \top$ then $\bsi_\bot ab = \top$,  otherwise $\bsi_\bot ab = \bot$.}
\end{itemize}

Some other properties of interest in this paper are 
\begin{align}
\bsi \top x = x \tag{LNP}\\
\bsi a(\bsi bc) = \bsi (b \bsi (ac)) \tag{Exchange Principle EP}\\
\bsi ab =1 \text{ if and only if } a\leq b   \tag{Ordering Property, OP}\\
\bsi a(\bsi ab) = \bsi ab  \tag{Iterative Boolean Law, IBL}\\
b \leq \bsi ab \tag{Consequent Boundary, CB}
\end{align}

Further, note that Tarski algebras are the same thing as implication algebras \cite{am5019,rh}. A few full dualities relating to classes of such algebras are known. One of this is a duality for finite Tarski sets \cite{csa,sclc2008} or covering approximation spaces. 

\begin{definition} 
A \emph{Tarski algebra} (or an \emph{implication algebra}) is an algebra of the form $S= \left\langle\underline{S}, \bsi , \top   \right\rangle$ of type $2, 0$  that satisfies (in the following,  \cite{rh})
\begin{align*}
\bsi \top a = a \tag{Left Neutrality, LNP}\\
\bsi aa = \top \tag{Identity Principle, IP}\\
\bsi a(\bsi bc) = \bsi (\bsi ab)(\bsi ac) \tag{T3}\\
\bsi (\bsi ab)b = \bsi (\bsi ba)a \tag{T4}
\end{align*}
\end{definition}

A join-semilattice order $\leq$ is definable in a IA as below:
\[(\forall a, b)\, a \leq b \leftrightarrow \bsi ab=\top ; \text{ the join is } a\vee b = \bsi(\bsi ab)b\]
\emph{Filters or deductive systems} of an IA $S$ are subsets $K\subseteq S$ that satisfy 
\[1\in K \, \&\, (\forall a, b)(a, \bsi ab\in K \longrightarrow b\in K)\]
The set of all filters $\mathcal{F}(S)$ is an algebraic, distributive lattice whose compact elements are all those filters generated by finite subsets of $S$. 

\section{Model of Rough Skeptic and Pessimistic Reasoning}

For concreteness, a minimal base model over which the theory will be invented is defined next. Generalizations to weaker order, antichains, and partial approximation operations are considered in a separate paper. 

\begin{definition}\label{rcon}
An algebra of the form ${B} \, =\, \left\langle \underline{B}, l, u, \vee,  \wedge, \bot, \top \right\rangle$ with $(\underline{B}, \vee, \wedge, \bot, \top )$ being a bounded lattice will be said to be a \emph{rough convenience lattice} (RCL) if the following conditions are additionally satisfied ($\leq$ is the associated lattice order, and the operations $l$ and $u$ are generalized lower and upper approximation operators respectively):
\begin{align}
(\forall x) x^{ll}= x^l \leq x \leq x^u \leq x^{uu}  \tag{lu1}\\
(\forall a, b) (a\leq b \longrightarrow a^l\leq b^l )    \tag{l-mo}\\
(\forall a, b) (a\leq b \longrightarrow a^u\leq b^u )    \tag{u-mo}\\
(\forall a, b) a^l \vee b^l \leq (a\vee b)^l \,\&\, a^u \vee b^u = (a\vee b)^u   \tag{lu2}\\
(\forall a, b) (a\wedge b)^l = a^l \wedge b^l \, \&\, (a\wedge b)^u \leq a^u \wedge b^u   \tag{lu3}\\
\top^u = \top \,\&\, \bot^{l} = \bot =\bot^u  \tag{topbot}
\end{align}
\end{definition}

\begin{proposition}
In Definition \ref{rcon}, \textsf{lu2} and \textsf{lu3} follow from \textsf{lu1, l-mo} and \textsf{u-mo}. 
\end{proposition}

The concept is weaker than that of a general abstract approximation space. Note that by default no relation between the lower and upper approximations are assumed. Further, nothing is assumed about negations or complementation. A special case of a rough convenience lattice is a set-HGOS under additional conditions. However, note that no granularity-related restrictions are imposed on a rough convenience lattice. 

An element $a\in\mathcal{B}$ will be said to be \emph{lower definite} (resp. \emph{upper definite}) if and only if $a^l = a$ (resp. $a^u = a$) and \emph{definite}, when it is both lower and upper definite. For possible concepts of rough objects the reader is referred to the paper \cite{am501}.
 
\begin{definition}
By a \emph{roughly consistent object} (RCO) (respectively \emph{lower RCO, upper RCO}) will be meant a set of elements of the \textsf{RCL} $\underline{B}$ of the form  $H = \{x ; (\forall b\in H)\,x^l =b^l\,\&\, x^u = b^u \}$ (respectively $H_l = \{x ; (\forall b\in H_l)\,x^l =b^l\}$ and $H_u = \{x ; (\forall b\in H_u)\, x^u = b^u \}$ ). The set of all roughly consistent objects is partially ordered by the set inclusion relation. Relative to this order, maximal roughly consistent objects will be referred to as \emph{rough objects}. Analogously, \emph{lower and upper rough objects} will be spoken of. The collection of such objects will respectively be denoted by $R(B)$, $R_l(B)$ and $R_u(B)$.
\end{definition}

\begin{proposition}
In a \textsf{RCL} $B$, every maximal roughly consistent object is an interval of the form $(x^l, x^u)$ for some $x\in B$s. The converse holds as well. 
\end{proposition}
\begin{proof}
The result follows from the monotonicity of $l$ and $u$, and \textsf{lu1}.  
\end{proof}

\begin{definition}\label{roorder}
By the \emph{rough order} $\Subset$ on $R(B)$, will be meant the relation $\Subset$ defined by $(x^l, x^u) \Subset (a^l, a^u)$ if and only if ($x^l \leq a^l$ and $x^u \leq a^u$).    
\end{definition}

It can be shown that $\Subset$ is a bounded partial lattice order on $R(B)$. The least element of $R(B)$ is $(\bot, \bot)$, and its greatest element is the interval $(\top, \top)$. The meet and join operations of $B$ induce partial lattice operations on $R(B)$ -- these are investigated separately.

\begin{definition}\label{ccaoa}
In a RCL $B$, let for any $a, b\in B$ \[a\cdot b :=a^l \wedge b^l \text{ and } a \otimes b = a^u \vee b^u. \] 

The operations $\cdot$ and $\otimes$ will respectively be referred to as \emph{cautious co-aggregation} \textbf{CCA}, and \emph{optimistic aggregation} \textbf{OA} respectively.  
\end{definition}

The operation $\cdot$ can as well be interpreted as a pessimistic co-aggregation. The appropriateness of the competing interpretations is dependent on the relation with the \textbf{OA} or on other context-specific features. 

\begin{theorem}\label{rclthm}
The \textbf{CCA} operation defined above satisfies all the following:
\begin{align*}
(\forall a, b)\,a\cdot b = b\cdot a \tag{Ccomm}\\
(\forall a, b, e)\,a\cdot (b \cdot e)= (a\cdot b)\cdot e \tag{Casso}\\
(\forall a, b)(a\leq b \longrightarrow a\cdot e \leq  b\cdot e) \tag{Cm}\\
(\forall a)\,a\cdot \bot = \bot \tag{Cb} 
\end{align*}
\end{theorem}
\begin{proof}
For any $ a, b\in B$, $a\cdot b = a^l \wedge b^l = b^l \wedge a^l = b\cdot a$. This proves \textsf{Ccomm}.

Associativity can be proved as follows. For any $ a, b, e\in B$, 
\begin{align*}
 a\cdot (b\cdot e) = a^l \wedge (b\cdot e)^l = \tag{by definition}\\
 = a^l \wedge (b^l \wedge e^l)^l =  a^l \wedge b^{ll} \wedge e^{ll} \tag{by lu3}\\
 = (a^l \wedge b^l) \wedge e^l = (a\cdot b)\cdot e  \tag{by lu1.}
\end{align*}

Monotonicity follows from the monotonicity of $l$.

Finally, $a\cdot \bot = a^l \wedge \bot^l = a^l \wedge \bot = \bot$.
\qed 
\end{proof}
 
\begin{theorem}\label{rcuthm}
Omitting the initial universal quantifiers,
\begin{align*}
a\otimes b = b\otimes a \tag{Acomm}\\
(a^{uu} = a^u \& b^{uu} = b^u \&  e^{uu} = e^u \longrightarrow a\otimes (b \otimes e)= (a\otimes b)\otimes e) \tag{wAasso1}\\
a\otimes((b\vee e)\otimes a) = ((a\vee b )\otimes c)\otimes c  \tag{wAsso2}\\
(a\leq b \longrightarrow a\otimes e \leq  b\otimes e) \tag{Am}\\
a\otimes \top = \top \tag{Ab} 
\end{align*}
\end{theorem} 
\begin{proof}
\textsf{Acomm} follows from the commutativity of $\vee$.

For any $ a, b, e\in B$, 
\begin{align*}
 a\otimes (b\otimes e) = a^u \vee (b\otimes e)^u = \tag{by definition}\\
 = a^u \vee (b^u \vee e^u)^u =  a^u \vee b^{uu} \vee e^{uu}, \tag{by lu2}\\
\text{However, } (a \otimes b) \otimes e = a^{uu}\vee b^{uu}\vee e^u  \tag{by lu2}\\
\text{So the premise of wAsso1 ensures it.}   
\end{align*}
 
Using the definition of $\otimes$ on the LHS and RHS of \textsf{wAsso2}, it can be seen that 
\begin{align*}
a\otimes((b\vee e)\otimes a) = a^u \vee ((b \vee e)^u\vee a^u)^u = \tag{by definition}\\
a^{uu}\vee b^{uu} \vee e^{uu} \tag{by Acomm, lu2, u-mo, lu1}\\
\text{Similarly, }((a\vee b )\otimes c)\otimes c = a^{uu}\vee b^{uu}\vee c^{uu} \vee c^{u} = \tag{by definition, lu2}\\
a^{uu}\vee b^{uu} \vee e^{uu} \tag{by u-mo}\\
\text{This proves wAsso2}.  \tag{by =}
\end{align*}

Monotonicity of $\otimes$ can be proved from the monotonicity of $\vee$ and that of $u$.
\qed 
\end{proof}

\begin{definition}\label{negrcl}
Two generalized negations are definable on a RCL, $B$ as follows:
\begin{align*}
\neg a = inf \{z : z\in B\, \& a\otimes z = \top\} \tag{addneg}\\
\sim a = sup \{z: z\in B \, \&\, a\cdot z = \bot\} \tag{mulneg}
\end{align*}
\end{definition}

These satisfy the properties specified in the next two theorems.

\begin{theorem}
\begin{align*}
(\forall a) \neg \neg a \leq a^u  \tag{WN3-N}\\
(\forall a, b) (a\leq b \longrightarrow \neg b \leq (\neg b)^u \leq (\neg a)^u)   \tag{WN2-N}\\
\neg \bot \leq \top \,\&\, \neq \top = \bot  \tag{WN1-N}
\end{align*}
\end{theorem}

\begin{proof}
$\neg \neg a \otimes \neg a = \top$. Therefore, 
\[(\neg \neg a)^u \vee (\neg a)^u = \top = (\neg a)^u \vee a^u\]
By the definition of $\neg$, it follows that $\neg \neg a \leq (\neg \neg a)^u \leq a^u $ as $(\neg a)^u$ is in both the equalities.

To see this suppose $\neg \neg a > a$, then $\neg \neg a \vee  a = \neg \neg a$ 
This implies $(\neg \neg a)^u \vee  a^u = (\neg \neg a)^u$, and 
$a^u \vee (\neg a)^u = \top$ contradicts  the definition of $\neg$.
This proves \textsf{WN3-N}.

Clearly, $(\neg b)^u \vee b^u = \top = (\neg a)^u \vee a^u $. If $a\leq b$ then $a^u \leq b^u$. So, $(\neg a)^u \vee a^u \vee b^u = \top$. This yields $(\neg a)^u \vee b^u = \top$. By the definition of $\neg$ and monotonicity of $\vee$, it is necessary that $(\neg b)^u \leq \vee (\neg a)^u$. This proves \textsf{WN2-N}.

If $a^u \vee \top^u = \top $ for some $a$, then $a$ can be any element of the universe because $\top^u = \top$. Of these the smallest is $\bot$. Therefore, $\neq \top = \bot $ holds. However, $\neg \bot$ is the infimum of the elements whose upper approximation is $\top$, and so $\neg \bot \leq \top$. \textsf{WN1-N} is thus proved.  
\qed
\end{proof}
\begin{theorem}
\begin{align*}
(\forall a, b) (a\leq b \longrightarrow (\sim b)^l \leq (\sim a)^l\leq \sim a)   \tag{WN2-S}\\
\bot \leq \sim \top \, \& \,  \top = \sim \bot   \tag{WN3-S}
\end{align*}
\end{theorem}

\begin{proof}
Clearly, $(\sim b)^l \wedge b^l = \bot = (\sim a)^l \wedge a^l $. If $a\leq b$ then $a^l \leq b^l$. So, $(\sim b)^l \wedge a^l  = \bot$. By the definition of $\sim $, this means $(\sim b)^l \leq (\sim a)^l \leq \sim a$, and proves \textsf{WN2-S}.

If $a^l \wedge \top^l = \bot $ for some $a$, then $a$ must necessarily be an element of the universe satisfying $a^l = \bot$. 
By definition, it is clear that in general, $a$ need not coincide with $\bot$.

If $a^l \wedge \bot^l = \bot $, then $a$ can be any element of the universe. The largest of these is $\top$. Therefore, $\sim \bot = \top$. This proves \textsf{WN3-S}.
\qed
\end{proof}

The above means that $\sim$ is a weak negation. The properties of the negation improve when the \textsf{RCL} satisfies a weak complementation $c$ that satisfies the conditions
\begin{align}
(\forall x) x^{cc} \leq x   \tag{c1}\\
(\forall x) x^c \wedge x = \bot   \tag{c2}
\end{align}
The above two conditions ensure that for any $a$, $\sim a \leq a^{lc}$.

\subsection{Implications}

Given the definitions of negation, and $*$-norms, a natural candidate for a definition of implication is given by the following equation 
\begin{equation}
\bsi_\neg ab = (\neg a)\otimes b \tag{Negimplication} 
\end{equation} 

\begin{theorem}\label{negimpl}
The operation $\bsi_\neg$ satisfies the properties: \textsf{FPA, IP, SPM, BC1, BC2, and BC3}. So it is an implication operation.
\end{theorem}
\begin{proof}
If $a\leq c$, then $\neg c \leq (\neg c)^u \leq (\neg a)^u $.
Therefore, $(\neg c)^u \vee b^u \leq (\neg a)^u \vee b^u$
From this it follows that $\bsi_\neg cb \leq \bsi_\neg ab$. So \textsf{FPA} holds.

Suppose $b \leq c$ then for any $a\in B$ 
$\bsi_\neg ab =(\neg a)^u \vee b^u $ and $\bsi_\neg ac = (\neg a)^u \vee c^u$ (by definition).
Monotonicity of $u$ ensures that $(\neg a)^u \vee b^u \leq (\neg a)^u \vee c^u$. Therefore, SPM must hold.

$\bsi_\neg \bot\bot =(\neg \bot)^u \vee \bot ^u = \top$ (by definition). So  \textsf{BC1} holds.

$\bsi_\neg \top \top = (\neg \top)^u\vee (\top)^u = \top$ (by definition). So \textsf{BC2} holds.

$\bsi_\neg \top \bot = (\neg \top)^u\vee (\bot)^u = (\bot)^u = \bot$. So \textsf{BC3} holds.

For any $a$, $\bsi_\neg aa = (\neg a)^u \vee (a)^u = \top$. So \textsf{IP} holds.
\qed 
\end{proof}

Other possibilities are 
\begin{align*}
\bsi_o ab = (\neg a)\vee b   \tag{negvee}\\
\bsi_\sim ab = (\sim a)\cdot b   \tag{simplication}\\
\bsi_s ab = (\sim a)\wedge b   \tag{simwed}
\end{align*}

Of these $\bsi_\sim$ is most interesting, and has the following properties:

\begin{theorem}\label{simpli}
In a \textsf{RCL} $B$, $\bsi_\sim$ satisfies \textsf{FPA, SPM, BC3, IBL} and converse of \textsf{CB}. 
\end{theorem}

\begin{proof}
FPA: Suppose $a\leq b$ for any $a, b\in B$. $\bsi_\sim bc = (\sim b)^l \wedge c^l$, and $\bsi_\sim ac = (\sim a)^l \wedge c^l$,
By \textsf{WN2-S}, it follows that $(\sim b)^l \wedge c^l \leq (\sim a)^l \wedge c^l$. This ensures \textsf{FPA}.

SPM: Suppose $a\leq b$ for any $a, b\in B$. $\bsi_\sim cb = (\sim c)^l \wedge b^l$, and $\bsi_\sim ca = (\sim c)^l \wedge a^l$.
Under the assumption $ (\sim c)^l \wedge a^l \leq (\sim c)^l \wedge b^l$. \textsf{SPM} follows from this.

BC3: $\bsi_\sim \top \top = (\sim \top)^l \wedge \top^l = \bot$. So \textsf{BC3} holds.

IBL: For any $a, b\in B$, $ia (\bsi ab) = (\sim a)^l \wedge ((\sim a )^l \wedge b^l)^l = $

$(\sim a)^l \wedge (\sim a)^{ll} \wedge b^{ll} = (\sim a)^l \wedge (\sim a)^{l} \wedge b^{l} = $

$ (\sim a)^{l} \wedge b^{l} = \bsi_\sim ab$. This proves \textsf{IBL}.

The converse of \textsf{CB} holds because $(\sim a)^l \wedge b^{l} \leq b$ is satisfied for all possible values of $a$.
\qed \end{proof}

It can be verified that \textsf{BC1, BC2,} and \textsf{IP} do not hold in general for $\bsi_\sim$.

The operation $\cdot$ can be naturally interpreted as a pessimistic or skeptical aggregation because it essentially selects a common part of two lower approximations (that are not restricted in their badness).  Two related operations are $\odot$ and $\times$ can be defined by (for any $a, b\in B$)
$a\odot b := a^l \vee b^l$ and $a\times b := a^u \wedge b^u$. The operation $\otimes$ on the other hand is optimistic at every stage of the reasoning process. First, the possibilistic upper approximation operators are used, and then the one that certainly contains the upper approximations is constructed. $\odot$ and $\times$ are essentially intermediate operations. 

\subsection{Concrete and Abstract Algebraic Models}

It is shown that a rough convenience lattice and closely related abstract algebraic systems have a far richer structure than is assumed in the literature. In concrete terms, every RCL can be naturally enhanced to the following algebraic system.

\begin{definition}\label{crclaa}
By a \emph{Concrete RCL Aggregation Algebra} (CRCLAA) will be meant an algebra of the form
\[{B} \, =\, \left\langle \underline{B}, \otimes, \cdot, \vee,  \wedge, l, u, \neg, \sim, \bot, \top \right\rangle\] with  $\left\langle \underline{B}, \vee,  \wedge, l, u, \bot, \top \right\rangle$ being a \textsf{RCL}, and the operations $\otimes, \cdots, \otimes, \cdot, \neg$, and $\sim$ are as defined in the previous subsections.
 \end{definition}

While the operations $\cdot$ and $\otimes$ are terms derived in the signature of the \textsf{RCL}, the other operations are defined by imposing a perspective on them. This suggests that an abstract property-based definition of an algebra of the same type may not be always equivalent to a CRCLAA. 
Additionally, it makes sense to retain the implications and omit negations.

\begin{definition}\label{arclana}
By an \emph{Abstract RCL Aggregation Negation Algebra} (CRCLANA) will be meant an algebra of the form \[{B} \, =\, \left\langle \underline{B}, \otimes, \cdot, \vee,  \wedge, l, u, \neg, \sim, \bot, \top \right\rangle\] that satisfies the following conditions:
\begin{align*}
\left\langle \underline{B}, \vee,  \wedge, l, u, \bot, \top \right\rangle \text{ is a RCL.}   \tag{rcl}\\
\cdot \text{ satisfies Ccomm, Casso, Cm, and Cb}.   \tag{cdotc}\\
\otimes \text{ satisfies Acomm, wAsso1, wAsso2, Am, and Ab}.   \tag{otimc}\\
\neg \text{ satisfies WN1-N, WN2-N, and WN3-N}.   \tag{negc}\\
\sim \text{ satisfies WN2-S, and WN3-S}. \tag{simc} 
\end{align*}
\end{definition}

\begin{definition}\label{arclaia}
An \emph{Abstract RCL Aggregation Implication Algebra} (CRCLAIA) shall be an algebra of the form ${B} \, =\, \left\langle \underline{B}, \otimes, \cdot, \vee,  \wedge, l, u, \bsi_\neg, \bsi_\sim \bot, \top \right\rangle$ that satisfies:
\begin{align*}
\left\langle \underline{B}, \vee,  \wedge, l, u, \bot, \top \right\rangle \text{ is a RCL.}  \tag{rcl}\\
\cdot \text{ satisfies Ccomm, Casso, Cm, and Cb}.   \tag{cdotc}\\
\otimes \text{ satisfies Acomm, wAsso1, wAsso2, Am, and Ab}.   \tag{otimc}\\
\bsi_\sim \text{ satisfies FPA, SPM, BC3, and IBL}.   \tag{imsc}\\
\bsi_\neg \text{ satisfies FPA, IP, SPM, BC1, BC2, and BC3}.   \tag{inegc}
\end{align*}
\end{definition}

The above allows the following interesting problems. It may be noted that the associated contexts in logic are not known because the defining conditions are not strong enough. 

\begin{problem}
Under what additional conditions are CRCLANA and CRCLAIA representable as concrete RCL aggregation algebras?  
\end{problem}

\section{Illustrative Examples}

In academic learning contexts, all stakeholders approximate concepts within their own frameworks, and perspectives \cite{am2022c}. However, the learning context admits of common languages of discourse -- it is very important that this be \emph{large and expressive} enough. In practical terms, this means that the admitted basic predicates or functions should be many in number, and be endowed with minimalist properties relative to what may be possible in associated contexts. Below an abstract and a concrete example are constructed. 

\subsection{Abstract Example}

Let $B= \{\bot, \top, a, b, c, e, f\}$ be endowed with the lattice order depicted in Figure \ref{rclex}. Suppose the lower and upper approximations are respectively\\ $\{(\bot,\bot), (\top,e), (a,c),(b, b), (c,c),(e,c),(f,\bot)\}$ and\\ $\{(\bot,\bot), (\top,\top), (a,a),(b, \top), (c,e),(e,e),(f,b)\}$ respectively. The operations $\otimes,$ $\cdot,$ $\neg, $ and $\sim$ are then computable as in the three tables, while the implications follow.
  
\begin{center}
 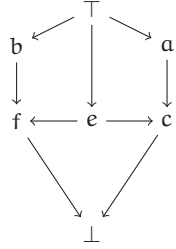
\begin{figure}
 \centering
\begin{tikzpicture}[node distance=0.5cm, auto]
\node (Alp) {$\top$};
\node (A01) [right of=Alp] {};
\node (A0) [right of=A01] {};
\node (A1) [left of=Alp] {};
\node (A) [left of=A1] {};
\node (Al) [below of=A0] {$a$};
\node (Alo) [below of=A] {$b$};
\node (B1) [below of=Al] {};
\node (B2) [below of=B1] {$c$};
\node (B3) [below of=Alo] {};
\node (B4) [below of=B3] {$f$};
\node (B5) [below of=Alp] {};
\node (B6) [below of=B5] {};
\node (B9) [below of=B6] {$e$};
\node (B10) [below of=B9] {};
\node (B11) [below of=B10] {};
\node (B15) [below of=B11] {$\bot$};
\draw[->] (Alp) to node {}(Al);
\draw[->] (Alp) to node {}(Alo);
\draw[->] (Al) to node {}(B2);
\draw[->] (Alo) to node {}(B4);
\draw[->] (B2) to node {}(B15);
\draw[->] (B4) to node {}(B15);
\draw[->] (Alp) to node {}(B9);
\draw[->] (B9) to node {}(B2);
\draw[->] (B9) to node {}(B4);
\end{tikzpicture}
\caption{Bounded Lattice}\label{rclex}
\end{figure}
\end{center}
\begin{table}[hbt]

\begin{minipage}{.33\linewidth}
\centering
\caption{$\otimes$-Table}\label{otimtab} 
\begin{tabular}{l|ccccccc}
\toprule
 $\otimes $ & $\bot $ & $\top $ & $a $ & $b $ & $ c$ & $e $ & $f $   \\
\midrule
$\bot $ & $\bot $ & $\top $ & $a $ & $\top $ & $ e$ & $e $ & $b $ \\
\midrule
$ \top$ & $\top $ & $\top $ & $\top $ & $\top $ & $\top $ & $\top $ & $\top $\\
\midrule
$a $ & $a $ & $ \top$ & $a $ & $\top $ & $\top $ & $\top $ & $\top $\\
\midrule
$ b$ & $\top $ & $\top $ & $\top $ & $\top $ & $\top $ & $\top $ & $\top $\\
\midrule
$ c$ & $e $ & $ \top$ & $\top $ & $\top $ & $e $ & $e $ & $\top $\\
\midrule
$ e$ &  $e $ & $ \top$ & $\top $ & $\top $ & $e $ & $e $ & $\top $\\
\midrule
$ f$ & $b $ & $\top $ & $\top $ & $\top $ & $\top $ & $\top $ & $b $\\
\bottomrule
\end{tabular}

\end{minipage}
\begin{minipage}{.33\linewidth}
\centering
\caption{$\cdot$-Table}\label{cdottab} 
\begin{tabular}{l|ccccccc}
\toprule
$\cdot $ & $\bot $ & $\top $ & $a $ & $b $ & $ c$ & $e $ & $f $   \\
\midrule
$\bot $ & $\bot $ & $\bot $ & $\bot $ & $\bot $ & $ \bot$ & $\bot $ & $\bot $ \\
\midrule
$\top $ & $\bot $ & $e $ & $c $ & $f $ & $c $ & $ c$ & $\bot $\\
\midrule
$a $ & $\bot $ & $ c$ & $ c$ & $\bot $ & $c $ & $c $ & $\bot $\\
\midrule
$ b$ & $\bot $ & $ f$ & $\bot $ & $ b$ & $\bot $ & $\bot $ & $\bot $\\
\midrule
$c $ & $\bot $ & $c $ & $c $ & $\bot $ & $c $ & $ c$ & $\bot $\\
\midrule
$e $ & $\bot $ & $c $ & $c $ & $\bot $ & $c $ & $ c$ & $\bot $\\
\midrule
$f $ & $\bot $ & $\bot $ & $\bot $ & $\bot $ & $\bot $ & $ \bot$ & $\bot $\\
\bottomrule
\end{tabular}
\end{minipage}
\begin{minipage}{.33\linewidth}
\centering
\caption{Negations}\label{negtab}
\begin{tabular}{l|ccccccc}
\toprule
 \textsf{Neg} & $\bot $ & $\top $ & $a $ & $b $ & $ c$ & $e $ & $f $   \\
\midrule
$\neg $ & $b $ &  $\bot $ & $\bot $ & $\bot $ & $ \bot$ & $\bot $ & $\bot $\\
\midrule
$ \sim$ & $\top $ & $f $ & $b $ & $c $ & $f $ & $f $ & $\top $\\
\bottomrule
\end{tabular}
\end{minipage} 
\end{table}

\subsection{Detection of Reasoning and Algorithm Bias}

Skeptical aggregation is often a feature of negative bias in human reasoning. Suppose a toxic person with decision-making powers is constrained by their environment from explicitly discriminating against specific groups of people. Then they are likely to discriminate by adopting additional distracting strategies and empty agendas. The effect of such practices can be analyzed through the aggregation strategies adopted. In fact, serious political analysts frequently try to do precisely that. 

The behavior of biased or defective algorithms is typically reflected in the data used, and produced by it (because the results produced at different stages are again a form of data). Analysis of empirical bias can possibly be deduced from the associated data sets. If it can be shown that bias is due to the algorithm learning from biased data, then it means that algorithm is not safe for the purpose. Such generalities can be analyzed with the proposed methodology.  

Many types of models are possible for information tables that are the result of systemic bias in the data collection process or due to external factors. These may be partly reflected in the data, and in such circumstances the methods invented in this research may be applicable. External factors may be taken into account through the approximation operators (about which the assumptions are left to the practitioner). The essence of the procedure is outlined below:

\begin{itemize}
 \item {Let $C_1, \ldots C_k$ be subsets of $B$ that potentially correspond to specific objects that are discriminated against.}
 \item {Let $E_1, \ldots E_k$ be subsets of $B$ that potentially correspond to specific objects that are unduly favored.}
 \item {Let $F_1, \ldots F_k$ be subsets of $B$ that potentially correspond to specific objects that help in bias determination}
 \item {It is assumed that a similar pattern of bias is maintained by the process.}
 \item {Compute $C_i \cdot F_i$, $C_i \otimes F_i$, $E_i \cdot F_i$, and $E_i \otimes F_i$}
 \item {Let the principal lattice filters generated by these be respectively $\Finv(C_i \cdot F_i)$, $\Finv(C_i \otimes F_i)$, $\Finv(E_i \cdot F_i)$, and $\Finv(E_i \otimes F_i)$.}
 \item {If $B$ is finite, compute the cardinalities of the principal lattice filters generated by the four.}
 \item {A simple measure of bias $\flat(C, E)$ is defined in Equation \ref{biase}. It will be referred to as the \emph{flat bias measure}.}
\end{itemize}

\begin{equation}
 \flat(C, E) = 1 -\dfrac{1}{k}\sum \dfrac{\text{Card}(\Finv(C_i \cdot F_i))}{\text{Card}(\Finv(E_i \cdot F_i))}
 \tag{biase}\label{biase}
\end{equation}

If it is certain that the co-aggregations are justified, then they can be used for the \emph{sharp bias measure} defined in Equation \ref{sharpe}

\begin{equation}
 \eth(C, E) = 1 -\dfrac{1}{k}\sum \dfrac{\text{Card}(\Finv(C_i \otimes F_i)) -\text{Card}(\Finv(C_i \cdot F_i))}{\text{Card}(\Finv(E_i \otimes F_i)) - \text{Card}(\Finv(E_i \cdot F_i))}
 \tag{sharpe}\label{sharpe}
\end{equation}

\section{Skeptical Aggregation and Rough Dependence}\label{skpd}

A theory of rough dependence, and associated measures for subclasses of granular rough sets (in the axiomatic sense) is invented by the present author in earlier papers \cite{am9411,am3930,am9501}. It concerns the extent to which an object depends on another expressed in terms of rough objects of different types. The representation is used in the context of contrasting it with that of probabilistic dependence. In fact, it is proved by her \cite{am9411} that the models of dependence based probability and models of rough dependence  do not share too many axioms. Subsequent research led to the invention of a theory of non-stochastic rough randomness and large-minded reasoners \cite{am23f}.

\begin{definition}
Let ${B} \, =\, \left\langle \underline{B}, \mathcal{G}, l, u, \vee,  \wedge, \bot, \top \right\rangle$ be a structure with $\underline{B}$ being a subset of a powerset $\wp(S)$, $\left\langle \underline{B}, l, u, \vee,  \wedge, \bot, \top \right\rangle$ being a \textsf{RCL}, $\wedge = \cap$, $\vee = \cup$, $\top = S$, and $\bot = \emptyset$, and $\mathcal{G} \subseteq B$ is a granulation on $B$ in the axiomatic sense \cite{am5586,am501,am240} ($t$ being a term function in the algebraic language of the RCL):
\begin{align*}
(\forall x \exists
a_{1},\ldots a_{r}\in \mathcal{G})\, t(a_{1},\,y_{2}, \ldots \,a_{r})=x^{l} \\
\tag{Weak RA, WRA} \mathrm{and}\: (\forall x)\,(\exists
a_{1},\ldots a_{r}\in \mathcal{G})\,t(a_{1},\,a_{2}, \ldots \,a_{r}) =
x^{u},\\
\tag{Lower Stability, LS}{(\forall a \in
\mathcal{G})(\forall {x\in \wp(\underline{S}) })\, ( a\subseteq x\,\longrightarrow\, a \subseteq x^{l}),}\\
\tag{Full Underlap, FU}{(\forall
x,\,a\in\mathcal{G})(\exists
z\in \wp(\underline{S}) )\, x \subset z,\,\&\, a \subset z\,\&\,z^{l}\, =\,z^{u}\, =\,z,}
\end{align*}
$B$ will then be referred to as a \emph{set granular RCL} (sGRCL).
\end{definition}

It is easy to see that all sGRCLs are set HGOS as well.

In a sGRCL $B$, if $\nu(B)$ is the collection of definite objects in some sense, 

\begin{definition}
The $\mathcal{G} \nu$-\emph{infimal degree of dependence} $\beta_{i \tau \nu}$ of $x$ on $z$ is defined by 
\begin{equation}
\beta_{i \mathcal{G} \nu} (x,\, z)\,=\,\inf _{\nu (B) }\,\bigcup \,\{g\,:\,g\in \mathcal{G} \, \&\,g\subseteq x \,\& \, g \subseteq z\}. 
\end{equation}
The infimum is over the $\nu(B)$ elements contained in the union.

The $\mathcal{G} \nu$-\emph{supremal degree of dependence} $\beta_{s \mathcal{G} \nu}$ of $x$ on $z$ is defined by
\begin{equation}
\beta_{s \mathcal{G} \nu} (x,\, z)\,=\,\sup _{\nu (B) }\,\bigcup \,\{g\,:\,g\in \mathcal{G} \,\&\, g\subseteq x \,\& \, g \subseteq z\}.                                                                                                     
\end{equation}
The supremum is over the $\nu(S)$ element containing the union.
\end{definition}

If unions of granules are always definite, then the two concepts coincide.

\begin{theorem}
In classical rough sets with $\mathcal{G}$ being the set of equivalence classes and $\nu(B)\,=\, \delta_{l}(B)$ - the set of lower definite elements, then 
\[\beta_{i}x z\,=\, x^{l}\, \cap\,z^ l \,=\,\beta_{s} x z \]
The converse of $(x\odot y\,=\, 0 \,\longrightarrow \, \beta_{i} x y \,=\, 0)$ is not true in general. 
\end{theorem}

The following proposition can be deduced
\begin{proposition}
In the context of classical rough sets, the degrees of rough dependence (relative to $\nu(B)$ being the set of lower definite elements) coincides with the skeptical aggregation operation.  
\end{proposition}

However, in slightly more general set-theoretical contexts, rough dependence does not coincide with skeptical aggregation. This follows from the above definition (additionally, readers may refer to Section 7 of the paper \cite{am9411}). 

\section{Directions}

It should be stressed that there is much scope for reducing the axioms assumed in studies on rough sets over residuated lattices or ortho-lattices \cite{cccd2018}. This research contributes to this broad project in the spirit of reverse mathematics that seeks a minimum of axioms for a result.  Dualities, somewhat related to recent results \cite{iewa2021}, for CRCLANA and CRCLAIA are of interest.

The terms \emph{pessimistic}, and \emph{optimistic} are used in different senses in the rough set and AIML literature. In the so-called multi-granulation studies \cite{qly} that concern contexts with multiple rough approximations (or multiple relations or granulations) on the same universe, it is used as an adjective for specific derived approximations. However, these are studied under other names in many older papers \cite{rac9,am909,mak2016}. Algebraic aspects are explored by the present author \cite{am909}, and others \cite{dls2022}. Three-way decision strategies are additionally studied \emph{to seek common ground while eliminating differences} \cite{zllmt2020}. Modal logic of the point-wise approximations in some of these contexts are explored in more recent work \cite{akvp18}. The present study is about models of aggregation and co-aggregation from a rough set view, and therefore it is not directly related to these as the idea necessarily involves systems of approximations from different sources (and is relative to at least two such sources). For example, by regulating the nature of the lower approximation, even extremely biased or bigoted views can be expressed by the aggregation $f$ mentioned earlier. As this is not really part of a multi source scenario, possible connections are research topics. Implication operations from a rough set perspective are studied in many related models such as quasi-boolean algebras \cite{ajmkc2016}. The results proved here show that many of the assumptions are not essential. A detailed study will appear separately.

In forthcoming papers, the semantics is extended to antichains of mutually distinct objects, building on earlier work of the present author \cite{am9114}. Further applications to concept modeling in education research and teaching contexts are areas of her ongoing research \cite{am2022c}.

\bibliographystyle{splncs04.bst}
\bibliography{algrough23+}

\end{document}